\documentclass[letterpaper]{article}
\usepackage{aaai}
\usepackage{times}
\usepackage{helvet}
\usepackage{courier}
\usepackage{amssymb,amsfonts,amsbsy,amsmath,amsthm}
\usepackage{graphicx}
\usepackage{multirow}
\usepackage{subfigure}
\usepackage{color}

\theoremstyle{plain}
\newtheorem{thm}{Theorem}
\newtheorem{theorem}[thm]{Theorem}

\newtheorem{definition}[thm]{Definition}

\newtheorem{lemma}[thm]{Lemma}

\newtheorem{proposition}[thm]{Proposition}

\newtheorem{example}[thm]{Example}

\newtheorem*{lemmaappendix}{Lemma}
\newtheorem*{theoremappendix}{Theorem}

\def\Var{\mathop{\rm Var}}

\begin{document}
%
\title{Symmetry-Aware Marginal Density Estimation}
\author{Mathias Niepert\\
Computer Science \& Engineering\\
University of Washington\\
Seattle, WA 98195-2350, USA\\
}
\maketitle
\begin{abstract}
\begin{quote}
The Rao-Blackwell theorem is utilized to analyze and improve the scalability of inference in large probabilistic models that exhibit symmetries. A novel marginal density estimator is introduced and shown both analytically and empirically to outperform standard estimators by several orders of magnitude. The developed theory and algorithms apply to a broad class of  probabilistic models including statistical relational models considered not susceptible to lifted probabilistic inference.
\end{quote}
\end{abstract}

\section{Introduction}

Many successful applications of artificial intelligence research are based on large probabilistic models. Examples include Markov logic networks~\cite{RD:2006}, conditional random fields~\cite{lafferty:2001} and, more recently, deep learning architectures~\cite{hinton:2006,bengio:2007,poon:2011}.
Especially the models one encounters in the statistical relational learning (SRL) literature often have joint distributions spanning millions of variables and features. Indeed, these models are so large that, at first sight, inference and learning seem daunting. For numerous of these models, however, scalable approximate and, to a lesser extend, exact inference algorithms do exist. Most notably, there has been a strong focus on lifted inference algorithms, that is, algorithms that group indistinguishable variables and features during inference. For an overview we refer the reader to~\cite{kersting:2012}. Lifted algorithms facilitate efficient inference in numerous large probabilistic models for which inference is NP-hard in principle.
 
We are concerned with the estimation of marginal probabilities based on a finite number of sample points. We show that the feasibility of inference and learning in large and highly symmetric probabilistic models can be explained with the Rao-Blackwell theorem from the field of statistics. The theory and algorithms do not directly depend on the syntactical nature of the relational models such as arity of predicates and number of variables per formula but only on the given automorphism group of the probabilistic model, and are applicable to classes of  probabilistic models much broader than the class of statistical relational models.
 
Consider an experiment where a coin is flipped $n$ times. While a frequentist would assume the flips to be i.i.d., a Bayesian typically makes the weaker assumption of \emph{exchangeability} -- that the probability of an outcome sequence only depends on the number of ``heads" in the sequence and not on their order. Under the non-i.i.d. assumption, a possible corresponding graphical model is the fully connected graph with $n$ nodes and high treewidth. The actual number of parameters required to specify the distribution, however, is only $n+1$, one for each sequence with $0 \leq k \leq n$ ``heads." Bruno de Finetti was the first to realize that such a sequence of random variables can be (re-)parameterized as a unique mixture of $n+1$ independent urn processes~\cite{finetti:1938}. 
It is this notion of a parameterization as a mixture of urn processes that is at the heart of our work. A direct application of de Finetti's results, however, is often impossible since not all  variables are exchangeable in realistic probabilistic models. 

Motivated by the intuition of exchangeability, we show that \emph{arbitrary} model symmetries allow us to re-paramterize the  distribution as a mixture of  independent urn processes where each urn consists of isomorphic joint assignments. Most importantly, we develop a novel Rao-Blackwellized estimator that implicitly estimates the fewer parameters of the simpler mixture model and, based on these, computes the marginal densities. We identify situations in which the application of the Rao-Blackwell estimator is tractable. In particular, we show that the Rao-Blackwell estimator is always  linear-time computable for single-variable marginal density estimation. By invoking the Rao-Blackwell theorem, we show that the mean squared error of the novel estimator is at least as small as that of the standard estimator and strictly smaller under non-trivial symmetries of the probabilistic model. Moreover, we prove that for estimates based on sample points drawn from a Markov chain $\mathcal{M}$, the bias of the Rao-Blackwell estimator is governed by the mixing time of the quotient Markov chain whose convergence behavior is superior to that of $\mathcal{M}$.

We present empirical results verifying that the Rao-Blackwell estimator always  outperforms the standard estimator by up to several orders of magnitude, irrespective of the model structure. Indeed, we show that the results of the novel estimator resemble those typically observed in lifted inference papers. For the first time such a performance is shown for an SRL model with a transitivity formula.

\section{Background}

We review some concepts from group and estimation theory.

\subsubsection{Group Theory}

A group is an algebraic structure ($\mathfrak{G}, \circ$), where $\mathfrak{G}$ is a set closed under a binary associative operation $\circ$ with an identity element and a unique inverse for each element. We often write $\mathfrak{G}$ rather than  ($\mathfrak{G}, \circ$). A permutation group acting on a set $\Omega$ is a set of bijections  $\mathfrak{g} : \Omega \rightarrow \Omega$ that form a group.
Let $\Omega$ be a finite set and let $\mathfrak{G}$ be a permutation group acting on $\Omega$. If $\alpha \in \Omega$ and $\mathfrak{g} \in \mathfrak{G}$ we write $\alpha^\mathfrak{g}$ to denote the image of $\alpha$ under $\mathfrak{g}$. 
A cycle $(\alpha_1\ \alpha_2\ ...\ \alpha_n)$
represents the permutation that maps $\alpha_1$ to $\alpha_2$, $\alpha_2$
to $\alpha_3$,..., and $\alpha_n$ to $\alpha_1$.
Every permutation can be written as a product of disjoint cycles. A generating set $R$ of a group is a subset of the group's elements such that every element of the group can be written as a product of finitely many elements of $R$ and their inverses.
 
We define a relation $\sim$ on $\Omega$ with $\alpha \sim \beta$ if and only if there is a permutation $\mathfrak{g} \in \mathfrak{G}$ such that $\alpha^\mathfrak{g} = \beta$. The relation partitions $\Omega$ into equivalence classes which we call orbits. We call this partition of $\Omega$ the orbit partition induced by $\mathfrak{G}$. We use the notation $\alpha^{\mathfrak{G}}$ to denote the orbit $\{\alpha^\mathfrak{g}\ |\ \mathfrak{g} \in \mathfrak{G}\}$ containing $\alpha$. For a permutation group $\mathfrak{G}$ acting on $\Omega$ and a sequence $\mathbf{A} = \langle \alpha_1, ..., \alpha_k\rangle \in \Omega^k$ we write $\mathbf{A}^g$ to denote the image $\langle {\alpha_1}^g, ..., {\alpha_k}^g\rangle$ of $\mathbf{A}$ under $\mathfrak{g}$. Moreover, we write  $\mathbf{A}^{\mathfrak{G}}$ to denote the orbit of the sequence $\mathbf{A}$. 



\subsubsection{Point Estimation}

 Let $s_1, ..., s_N$ be $N$ sample points drawn from some distribution $P$. An estimator $\hat{\theta}_N$ of a parameter $\theta$ is a function of $s_1, ..., s_N$. The bias of an estimator is defined by $\mathtt{bias}(\hat{\theta}_N) := \mathbb{E}[\hat{\theta}_N -\theta]$ and the variance by $\Var(\hat{\theta}_N) :=  \mathbb{E}[(\hat{\theta}_N - \mathbb{E}(\hat{\theta}_N))^2]$, where $\mathbb{E}$ is the expectation with respect to $P$, the distribution that generated the data. We say that $\hat{\theta}_N$ is unbiased if $\mathtt{bias}(\hat{\theta}_N) = 0$. The quality of an estimator is often assessed with the mean squared error (MSE) defined by $\mbox{MSE}[\hat{\theta}_N] := \mathbb{E}[(\hat{\theta}_N - \theta)^{2}] = \Var(\hat{\theta}_N) + \mathtt{bias}(\hat{\theta}_N)^2$.
 
 \begin{theorem}[Rao-Blackwell]
Let $\hat{\theta}$ be an estimator with $\mathbb{E}[\hat{\theta}^2] < \infty$ and $T$ a sufficient statistic both for $\theta$, and let $\hat{\theta}^{*} := \mathbb{E}[\hat{\theta} \mid T]$. Then, $\textit{MSE}[\hat{\theta}^{*}] \leq \textit{MSE}[\hat{\theta}]$. Moreover, $\textit{MSE}[\hat{\theta}^{*}] < \textit{MSE}[\hat{\theta}]$ unless $\hat{\theta}^{*}$ is a function of $\hat{\theta}$.
\end{theorem}
 
\subsubsection{Finite Markov chains}

A finite Markov chain $\mathcal{M}$ defines a random walk  on elements of a finite set $\Omega$. For all $x, y \in \Omega$, $Q(x, y)$ is the chain's probability to transition from $x$ to $y$, and $Q^t(x, y)=Q^{t}_{x}(y)$ the probability of being in state $y$ after $t$ steps if the chain starts at $x$. A Markov chain is irreducible if for all $x, y \in \Omega$ there exists a $t$ such that $Q^t(x, y) > 0$ and aperiodic if for all $x \in \Omega$, $\mathsf{gcd}\{t \geq 1\ |\ Q^t(x, x) > 0\} = 1$. An irreducible and aperiodic chain converges to its unique stationary distribution and is called ergodic.

The total variation distance $d_{\mathsf{tv}}$ of the Markov chain from its stationary distribution $\pi$ at time $t$ with initial state $x$ is defined by 
$$\ d_{\mathsf{tv}}(Q^{t}_{x},\pi) = \frac{1}{2} \sum_{y \in \Omega} |Q^{t}(x, y) - \pi(y)|.$$ For $\varepsilon > 0$, let $\tau_x(\varepsilon)$ denote the least value $T$ such that $d_{\mathsf{tv}}(Q^{t}_{x},\pi) \leq \varepsilon$ for all $t \geq T$. The  \emph{mixing time} $\tau(\varepsilon)$ is defined by $\tau(\varepsilon) = \max\{\tau_x(\varepsilon)\ |\ x \in \Omega\}$.

\section{Related Work}

There are numerous lifted inference algorithms such as lifted variable elimination~\cite{poole:2003}, lifted belief propagation~\cite{singla2008lifted,kersting2009counting}, first-order knowledge compilation~\cite{broeck:2011}, and lifted variational inference~\cite{choi:2012}. Probabilistic theorem proving applied to a clustering of the relational model was used to lift the Gibbs sampler~\cite{venugopal:2012}. Recent work exploits automorphism groups of probabilistic models for more efficient probabilistic inference~\cite{hung:2012,niepert2012omcmc}. Orbital Markov chains~\cite{niepert2012omcmc} are a class of Markov chains that implicitly operate on the orbit partition of the assignment space and do not invoke the Rao-Blackwell theorem. 

Rao-Blackwellized (RB) estimators have been used for inference in Bayesian networks~\cite{doucet:2000,bidyu:2007} and latent Dirichlet allocation~\cite{teh:2006} with application in robotics~\cite{stachniss:2005} and activity recognition~\cite{bui:2002}. The RB theorem and estimator are important concepts in statistics~\cite{gelfand:1990,casella:1996}. 

\section{Symmetry-Aware Point Estimation}

An automorphism group of a probabilistic model is a group whose elements are permutations of the probabilistic model's random variables $\mathbf{X}$ that leave the joint distribution $P(\mathbf{X})$ invariant. There is a growing interest in computing and utilizing automorphism groups of probabilistic models for more efficient inference algorithms~\cite{hung:2012,niepert2012omcmc}. The line of research is primarily motivated  by the highly symmetric nature of statistical relational models and provides a complementary view on lifted probabilistic inference. 
Here, we will not be concerned with deriving automorphism groups of probabilistic models but with developing algorithms that utilize these permutation groups for efficient marginal density estimation. Hence, we always assume a given automorphism group $\mathfrak{G}$ of the probabilistic model under consideration.

We begin by deriving a re-parameterization of the joint distribution in the presence of symmetries that generalizes the mixture of independent urn processes parameterization for finitely exchangeable variables~\cite{diaconis:1980}. All random variables are assumed to be discrete.

Let $\mathbf{X}=\langle X_1, ..., X_n\rangle$ be a finite sequence of discrete random variables with joint distribution $P(\mathbf{X})$, let $\mathfrak{G}$ be an automorphism group of $\mathbf{X}$, and let $\mathcal{O}$ be an orbit partition of the assignment space induced by $\mathfrak{G}$. Please note that for any $\mathbf{x}, \mathbf{x'} \in O \in  \mathcal{O}$ we have $P(\mathbf{x})=P(\mathbf{x'})$. For a subsequence $\mathbf{\hat{X}}$ of $\mathbf{X}$ and an orbit $O \in \mathcal{O}$ we write $P(\mathbf{\hat{X}}=\mathbf{\hat{x}}\ |\ O)$ for the marginal density $P(\mathbf{\hat{X}}=\mathbf{\hat{x}})$ conditioned on $O$. Thus,  $$P(\mathbf{\hat{X}}=\mathbf{\hat{x}}\ |\ O) = \frac{1}{|O|} \sum_{\mathbf{x} \in O} \mathbb{I}_{\{\mathbf{x}\langle\mathbf{\hat{X}}\rangle=\mathbf{\hat{x}}\}},$$ where $\mathbb{I}$ is the indicator function and $\mathbf{x}\langle\mathbf{\hat{X}}\rangle$ the assignment within $\mathbf{x}$ to the variables in the sequence $\mathbf{\hat{X}}$.
We can now (re-)parameterize the marginal density as a mixture of independent orbit distributions $$P(\mathbf{\hat{X}}=\mathbf{\hat{x}}) = \sum_{O \in \mathcal{O}}  P(\mathbf{\hat{X}}=\mathbf{\hat{x}}\ |\ O )P(O), $$where $P(O)=\sum_{\mathbf{x}\in O}P(\mathbf{X}=\mathbf{x})$. For instance, the joint distribution of the Markov logic network in Figure~\ref{fig:lumping}(a) can be parameterized as a mixture of the distributions for the $10$ orbits depicted in Figure~\ref{fig:lumping}(c). 

Let us first recall the standard estimator used in most sampling approaches. After collecting $N$ sample points $s_1, ..., s_N$ the standard estimator for the marginal density $\theta := P(\mathbf{\hat{X}}=\mathbf{\hat{x}})$ is defined as
\begin{equation}
\hat{\theta}_{N} := \frac{1}{N}\sum_{i = 1}^{N} \mathbb{I}_{\{s_i\langle\mathbf{\hat{X}}\rangle=\mathbf{\hat{x}}\}}.
\end{equation}

Now, the \emph{symmetry-aware Rao-Blackwell} estimator for $N$ sample points $s_1, ..., s_N$ is defined as
\begin{equation}
\hat{\theta}^{\mathtt{rb}}_{N} := \frac{1}{N}\sum_{i = 1}^{N} P(\mathbf{\hat{X}}=\mathbf{\hat{x}}\ |\ {s_i}^{\mathfrak{G}}),
\end{equation} 
where $\mathfrak{G}$ is the given automorphism group that induces $\mathcal{O}$.
\begin{figure}[t!]
\begin{center}
\includegraphics[width=0.35\textwidth]{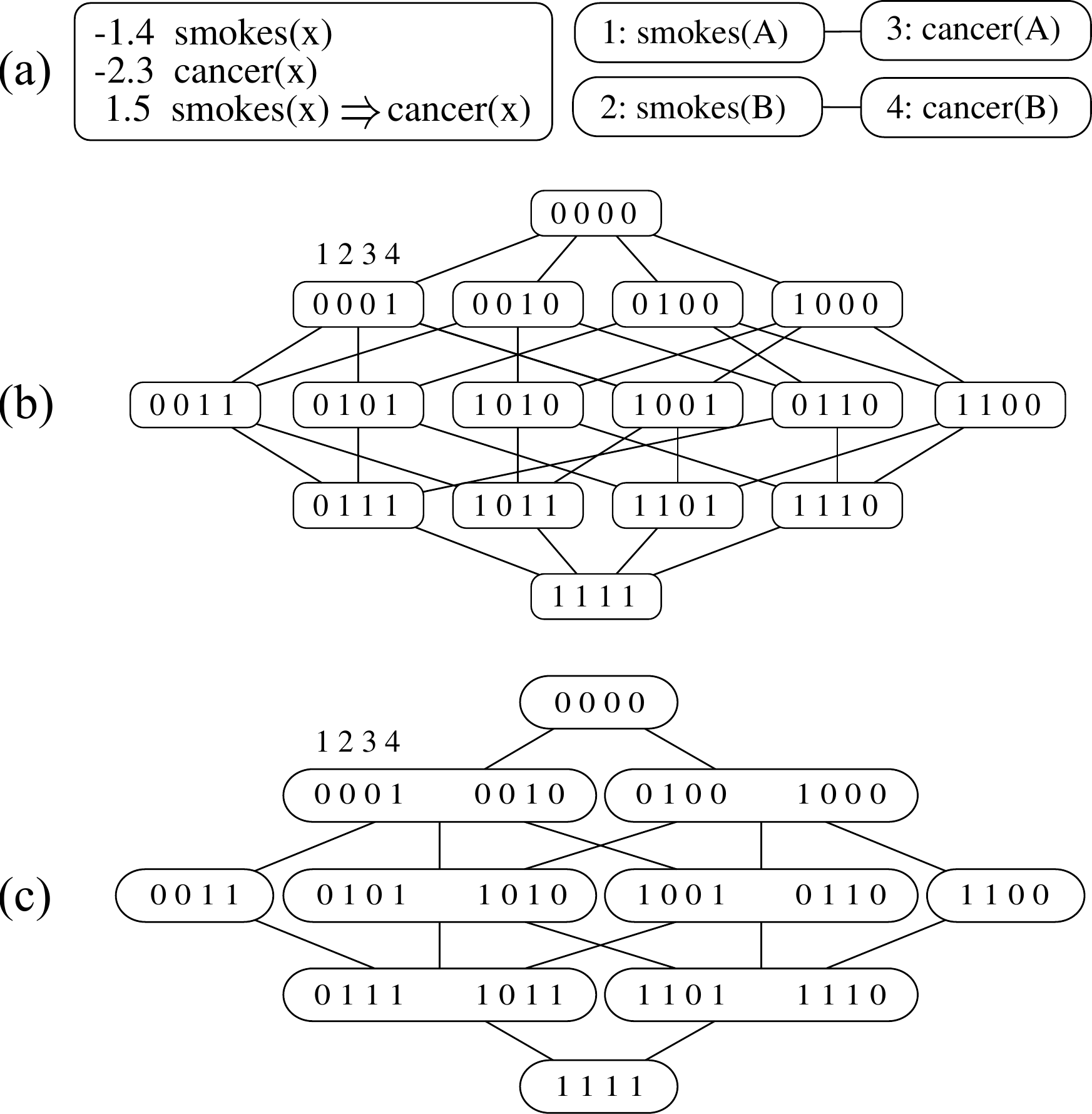}
\caption{Illustration of the orbit partition of the assignment space induced by the renaming automorphism group \{(A\ B), ()\}. A renaming automorphism is a permutation of constants that forms an isomorphism between two graphical models. (a) An MLN with three formulas and the grounding for two constants A and B; (b) the state space of a Gibbs chain with non-zero transitions indicated by lines and without self-arcs; (c) the lumped state space of the quotient Markov chain which has 10 instead of 16 states. The joint distribution can be expressed as a mixture of draws from the orbits.}
\label{fig:lumping}
\end{center}
\end{figure} 
Hence, the \emph{unbiased} Rao-Blackwell estimator integrates out the joint assignments of each orbit. We will prove that the mean squared error of the  Rao-Blackwell estimator is less than or equal to that of the standard estimator. First, however, we want to investigate under what conditions we can efficiently compute the conditional density of equation (2). To this end, we establish a connection between the orbit of the subsequence $\mathbf{\hat{X}}$ under the automorphism group $\mathfrak{G}$ and the orbit partition of the assignment space induced by $\mathfrak{G}$\footnote{Please note the \emph{two different types} of orbit partitions discussed here. One results from $\mathfrak{G}$ acting on the assignment space the other from $\mathfrak{G}$ acting on sequences of random variables.}.

\begin{definition}
Let $\mathbf{X}$ be a finite sequence of random variables with joint distribution $P(\mathbf{X})$, let $\mathfrak{G}$ be an automorphism group of $\mathbf{X}$,  let $\mathbf{\hat{X}}$ be a subsequence of $\mathbf{X}$, let $\mathsf{Val}(\mathbf{X})$ be the assignment space of $\mathbf{X}$, and let $s \in \mathsf{Val}(\mathbf{X})$. The \emph{orbit Hamming weight} of $s$ with respect to the marginal assignment $\mathbf{\hat{X}}=\mathbf{\hat{x}}$ is defined by $$\mathsf{H}_{\mathbf{\hat{X}}=\mathbf{\hat{x}}}^{\mathfrak{G}}(s) := \sum_{\mathbf{A} \in {\mathbf{\hat{X}}}^\mathfrak{G}}\mathbb{I}_{\{s \langle \mathbf{A} \rangle = \mathbf{\hat{x}}\}}.$$
\end{definition}

Based on this definition, we state a lemma which allows us to compute the  density of equation (2) in closed form, without having to enumerate all of the orbit's elements.

\begin{lemma}
\label{lemma-closed-form}
Let $\mathbf{X}$ be a finite sequence of random variables with joint distribution $P(\mathbf{X})$, let $\mathfrak{G}$ be an automorphism group of $\mathbf{X}$,  let $\mathbf{\hat{X}}$ be a subsequence of $\mathbf{X}$, and let $s \in \mathsf{Val}(\mathbf{X})$. Then,  
$$P(\mathbf{\hat{X}}=\mathbf{\hat{x}}\ |\ s^{\mathfrak{G}}) = \frac{\mathsf{H}_{\mathbf{\hat{X}}=\mathbf{\hat{x}}}^{\mathfrak{G}}(s)}{|\mathbf{\hat{X}}^\mathfrak{G}|} = \mathbb{E}[ \hat{\theta}_{N}\ |\ \mathsf{H}_{\mathbf{\hat{X}}=\mathbf{\hat{x}}}^{\mathfrak{G}}(s) ].$$
\end{lemma}

The following example demonstrates the application of the lemma to the special case of single-variable marginal density estimation for the MLN in Figure~\ref{fig:lumping}.

\begin{example}
Let us assume we want to estimate the marginal density $P($smokes(A)=$1)$ of the MLN in Figure~\ref{fig:lumping}(a). Since $\mathfrak{G}$ = \{(smokes(A) smokes(B))(cancer(A) cancer(B)), ()\} we have that $\langle$smokes(A)$\rangle^\mathfrak{G}$=\{$\langle$smokes(A)$\rangle, \langle$smokes(B)$\rangle\}$. Given the sample point $s = \langle 1, 0, 1, 0\rangle$ we have that $\mathsf{H}_{\langle\mbox{smokes(A)}\rangle=\langle 1 \rangle}^{\mathfrak{G}}(s) = 1$ and $P($smokes(A)=$1\ |\ s^{\mathfrak{G}}) = \frac{1}{2}$.
\end{example}

Thus, given a sample point $s$, the marginal density conditioned on an orbit of the assignment space is computable in closed form using the \emph{orbit Hamming weight} of $s$ with respect to the marginal assignment since it is a \emph{sufficient statistic} for the marginal density. 
If the probabilistic model exhibits symmetries, then the Rao-Blackwell estimator's MSE is less than or equal to that of the standard estimator. 

\begin{theorem}
\label{theorem-rao-blackwell}
Let $\mathbf{X}$ be a finite sequence of random variables with joint distribution $P(\mathbf{X})$, let $\mathfrak{G}$ be an automorphism group of $\mathbf{X}$ given by a generating set $R$, let $\mathbf{\hat{X}}$ be a subsequence of $\mathbf{X}$, and let $\theta := P(\mathbf{\hat{X}}=\mathbf{\hat{x}})$ be the marginal density to be estimated. The Rao-Blackwell estimator $\hat{\theta}^{\mathtt{rb}}_{N}$ has the following properties:
\begin{enumerate}
{\setlength\itemindent{12pt} \item[(a)] Its worst-case time complexity is $O(R|\mathbf{\hat{X}}^\mathfrak{G}|\hspace{-0.5mm} +\hspace{-0.5mm} N |\mathbf{\hat{X}}^\mathfrak{G}|)$; }
{\setlength\itemindent{12pt} \item[(b)] $\mbox{MSE}[\hat{\theta}^{\mathtt{rb}}_{N}] \leq \mbox{MSE}[\hat{\theta}_N]$. }
\end{enumerate} 
The inequality of (b) is strict if there exists a joint assignment $s$ with non-zero density and $0 < \mathsf{H}_{\mathbf{\hat{X}}=\mathbf{\hat{x}}}^{\mathfrak{G}}(s) <|\mathbf{\hat{X}}^{\mathfrak{G}}| > 1$.
\end{theorem}

For single-variable density estimation the worst-case time complexity of the Rao-Blackwell estimator is $O(R|\mathbf{X}| + N|\mathbf{X}|)$ and, therefore, linear both in the number of variables and the number of sample points. For most symmetric models, the inequality of Theorem~\ref{theorem-rao-blackwell}(b) is strict and the Rao-Blackwell estimator  outperforms the standard estimator, a behavior we will verify empirically.

Please note that in the special case of single-variable marginal density estimation the RB estimator is identical to the estimator that averages the identically distributed variables located in the same orbit. The advantages of utilizing the Rao-Blackwell theory are (1) it directly provides conditions for which the inequality of Theorem~\ref{theorem-rao-blackwell}(b) is strict; (2) it generalizes the single-variable case to marginals spanning multiple variables; (3) it allows us to investigate the \emph{completeness} of an estimator with respect to a given automorphism group; and (4) it provides the link to the quotient Markov chain in the MCMC setting and its superior convergence behavior presented in the following section.


The Rao-Blackwell estimator is unbiased if the drawn sample points are independent. Since it is often only practical to collect sample points from a Markov chain, the bias for a finite number of $N$ points will depend on the chain's mixing behavior. We will show that if there are non-trivial model symmetries and if we are using the Rao-Blackwell estimator, we only need to worry about the mixing behavior of the Markov chain whose state space is the orbit partition.

\section{Symmetry-Aware MCMC}
\label{symmetry-aware-section}

Whenever we collect sample points from a Markov chain, the efficiency of an estimator is influenced  by (a) the mixing behavior of the Markov chain and (b) the variance of the estimator under the assumption that the Markov chain has reached stationarity, that is, the asymptotic variance \cite{neal:2004}. That the Rao-Blackwell estimator's asymptotic variance is at least as low as that of the standard estimator is a corollary of Theorem~\ref{theorem-rao-blackwell}. We show that the same is true for the bias that is caused by the fact that we collect a finite number of sample points from Markov chains which never \emph{exactly} reach stationarity.

A lumping of a Markov chain is a partition of its state space which is possible under certain conditions on the transition probabilities of the original Markov chain~\cite{buchholz:1994,derisavi:2003}. 

\begin{definition}
Let $\mathcal{M}$ be an ergodic Markov chain with transition matrix $Q$, stationary distribution $\pi$, and state space $\Omega$, and let $\mathcal{C} = \{C_1, ..., C_n\}$ be a partition of the state space. If for all $C_i, C_j \in \mathcal{C}$ and all ${s_i}', {s_i}'' \in C_i$ $$Q'(C_i, C_j) := \sum_{{s_j} \in C_j} Q({s_i}',{s_j})=\sum_{{s_j} \in C_j} Q({s_i}'', {s_j})$$ then  we say that $\mathcal{M}$ is ordinary lumpable with respect to $\mathcal{C}$. If, in addition, $\pi(s_i')=\pi(s_i'')$ for all $s_i', s_i'' \in C_i$ and all $C_i \in \mathcal{C}$ then $\mathcal{M}$ is exactly lumpable with respect to $\mathcal{C}$. The Markov chain $\mathcal{M'}$ with state space $\mathcal{C}$ and transition matrix $Q'$ is called the quotient chain of $\mathcal{M}$ with respect to $\mathcal{C}$.
\end{definition}

Every finite ergodic Markov chain is exactly lumpable with respect to an orbit partition of its state space. The following theorem states this and the convergence behavior of the quotient Markov chain in relation to the original Markov chain (cf. \cite{boyd:2005}).

\begin{proposition}
\label{proposition-lumpable}
Let $\mathcal{M}$ be an ergodic Markov chain and let $\mathcal{O}$ be an orbit partition of its state space. Then, the Markov chain $\mathcal{M}$ is exactly lumpable with respect to $\mathcal{O}$. If $\mathcal{M}$ is reversible, then the quotient Markov chain $\mathcal{M'}$ with respect to $\mathcal{O}$ is also reversible. Moreover, the mixing time of $\mathcal{M'}$ is smaller than or equal to the mixing time of $\mathcal{M}$.
\end{proposition}

\begin{example}
Figure~\ref{fig:lumping}(b) depicts the state space of the Gibbs chain for the MLN shown in Figure~\ref{fig:lumping}(a). The constants renaming automorphism group  \{(A B), ()\} acting on the sets of constants leads to the automorphism group \{(smokes(A) smokes(B))(cancer(A) cancer(B)), ()\} on the ground level. This permutation group acting on the state space of the Gibbs chain induces an orbit partition which is the state space of the quotient Markov chain (see Figure~\ref{fig:lumping}(c)).
\end{example}

\begin{figure*}[t!]
     \begin{center}
        \subfigure[The asthma-smokes-cancer MLN with $50$ people and $10\%$ evidence.]{%
            \label{fig:first}
            \includegraphics[width=0.285\textwidth]{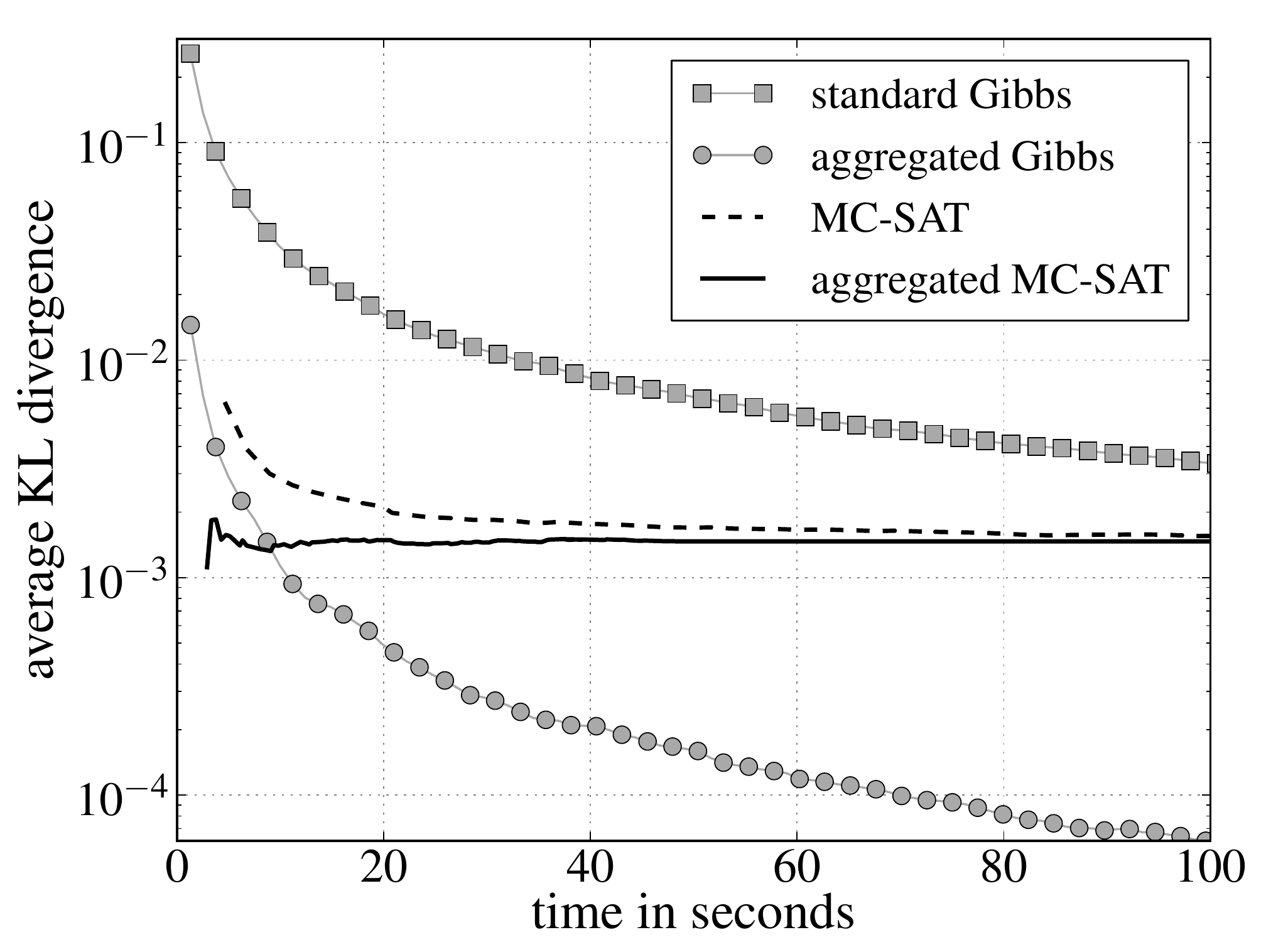}
        }%
        \hspace{5mm}
        \subfigure[The smokes-cancer MLN with $50$ people and $10\%$ evidence.]{%
           \label{fig:second}
           \includegraphics[width=0.285\textwidth]{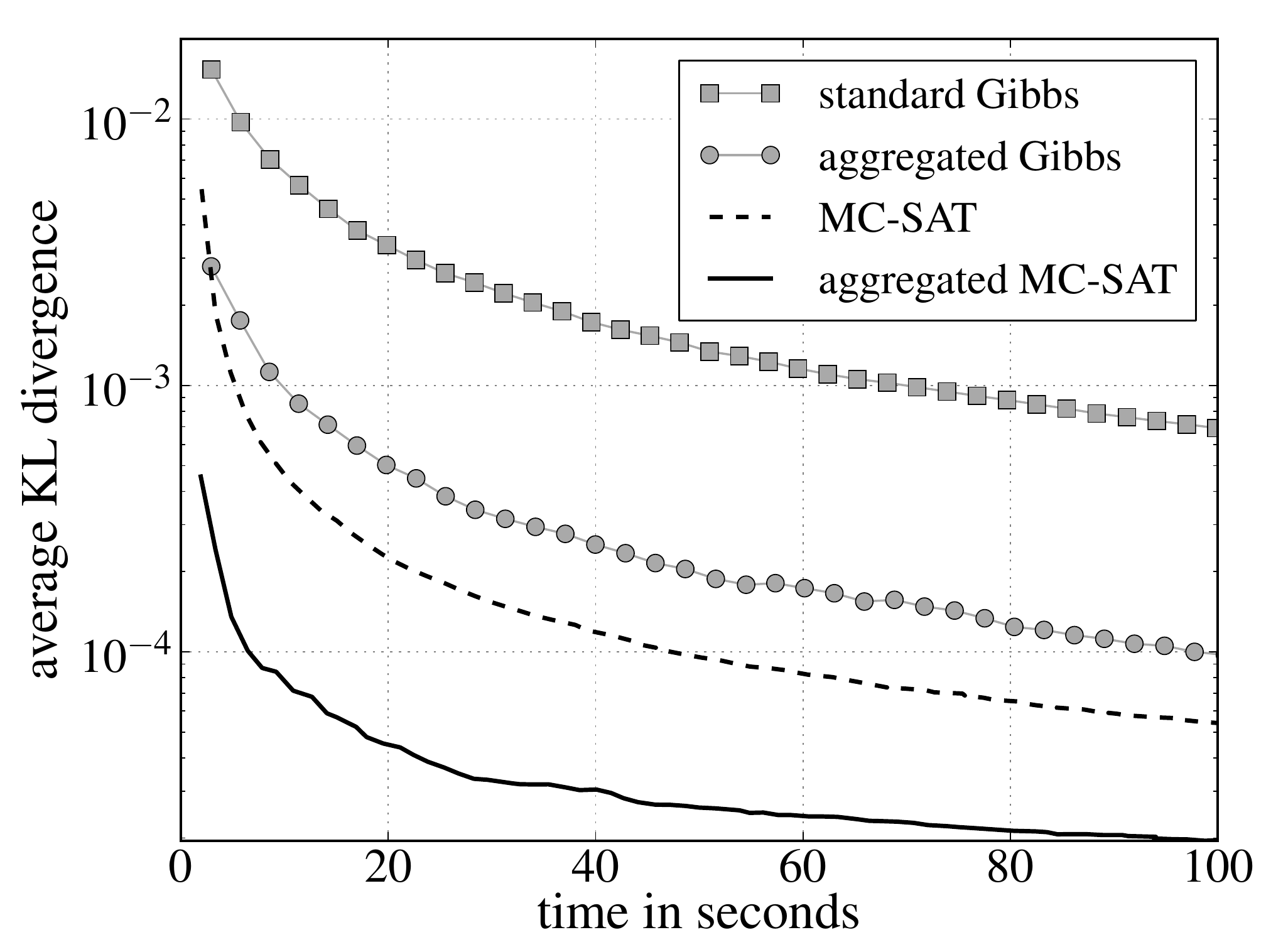}
        }
        \hspace{5mm}
        \subfigure[The smokes-cancer MLN with $50$ people, no evidence, and transitivity.]{%
            \label{fig:third}
            \includegraphics[width=0.285\textwidth]{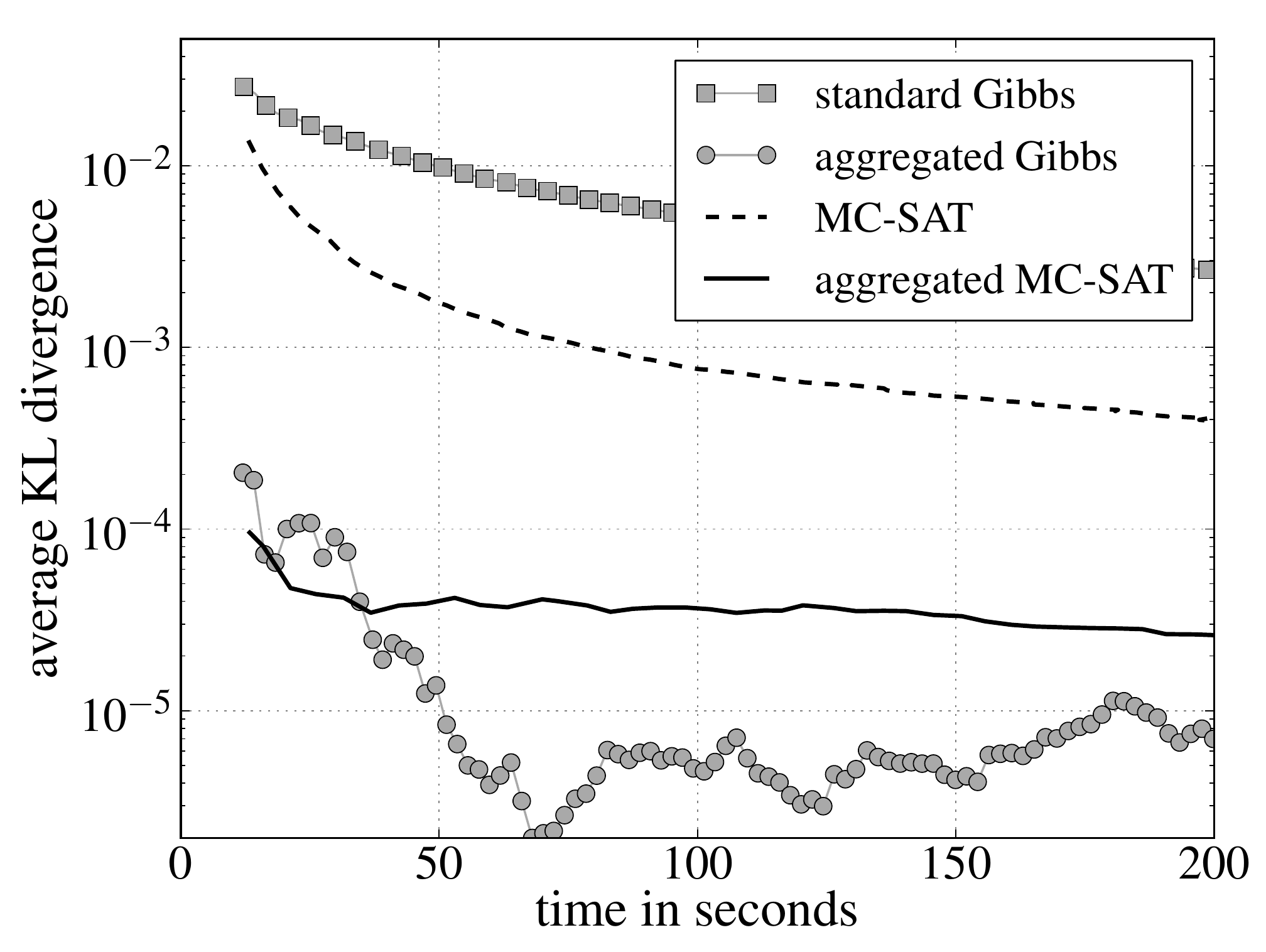}
        }
        \\ 
        \subfigure[The smokes-cancer MLN with $50$ people, $10\%$ evidence, and transitivity.]{%
            \label{fig:fourth}
            \includegraphics[width=0.285\textwidth]{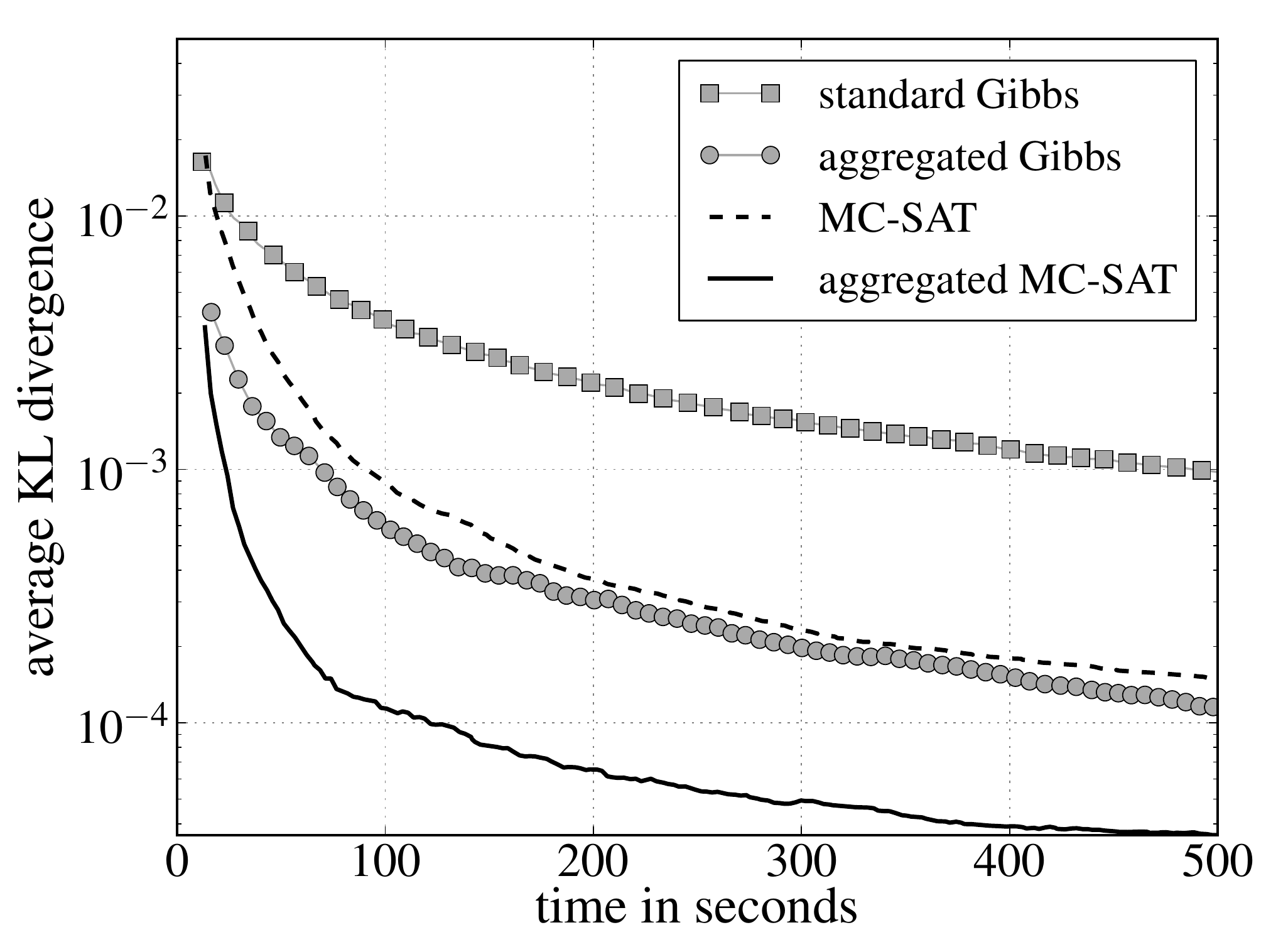}
        }%
        \hspace{5mm}
        \subfigure[The $2$-coloring $100\times 100$ grid model with weight $0.2$.]{%
            \label{fig:fifth}
            \includegraphics[width=0.285\textwidth]{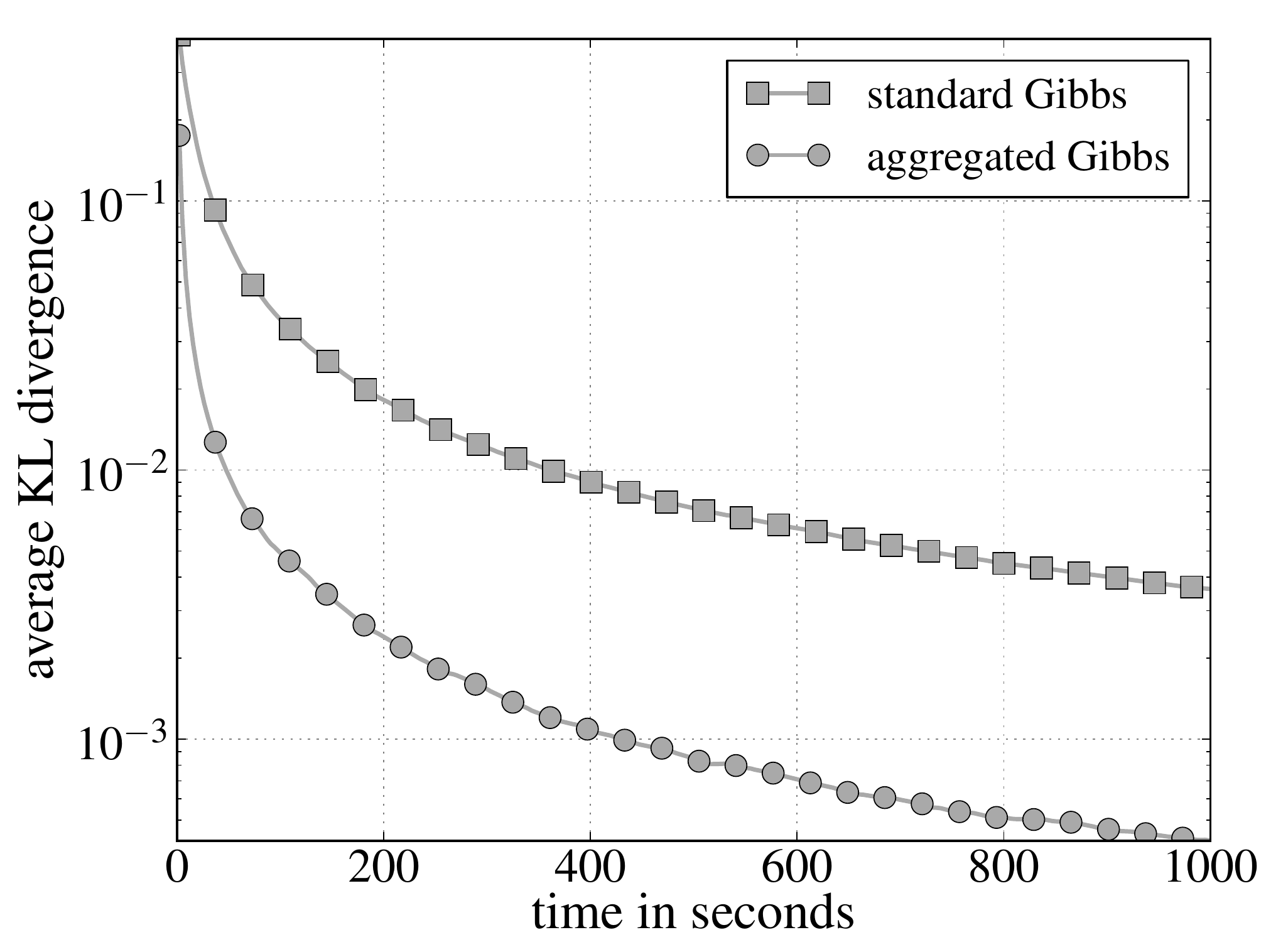}
        }%
        \hspace{5mm}
        \subfigure[The $2$-coloring  $100\times 100$ grid model with hard constraints.]{%
            \label{fig:sixth}
            \includegraphics[width=0.285\textwidth]{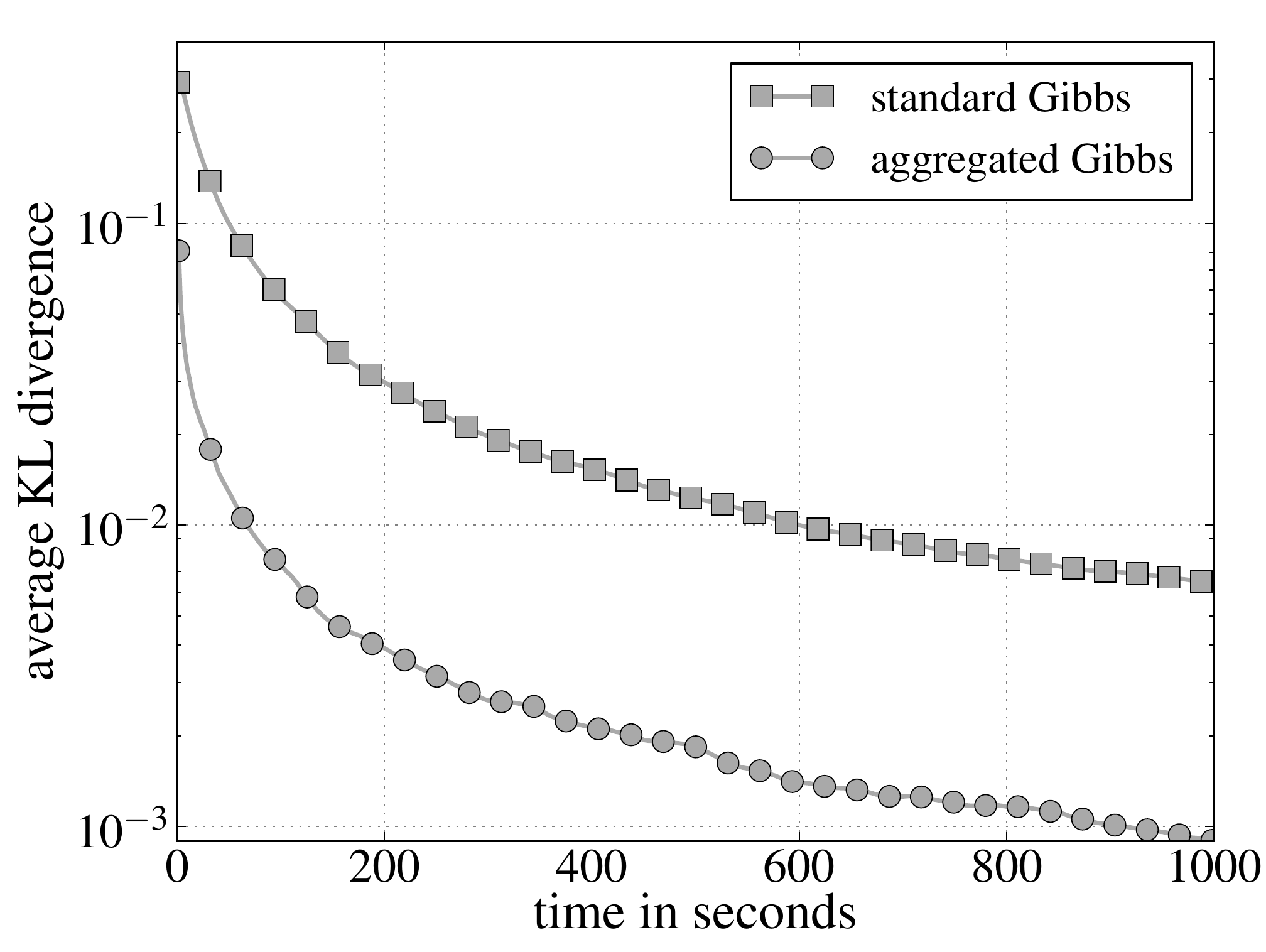}
        }%
    \caption{\label{fig:MLNs-comparative}Plots of average KL divergence versus time in seconds of the two MCMC algorithms with the standard estimator (standard) and the Rao-Blackwell estimator (aggregated) for various probabilistic models.}      
\end{center}
\end{figure*}

The explicit construction of the state space of a quotient Markov chain is  intractable. Given an automorphism group $\mathfrak{G}$, merely counting the number of equivalence classes of the orbit partition of the assignment space induced by $\mathfrak{G}$ is known to be a $\#$P-complete problem~\cite{goldberg:2001}. 
Nevertheless, if the Rao-Blackwell estimator is utilized, one can draw the sample points  from the original Markov chain while analyzing the convergence behavior of the quotient Markov chain of the original chain. 

\begin{theorem}
\label{theorem-mcmc}
\looseness=-1 Let $\mathbf{X}$ be a finite sequence of random variables with joint distribution $P(\mathbf{X})$, let $\mathcal{M}$ be an ergodic Markov chain with stationary distribution $P$, and let $\mathcal{O}$ be an orbit partition of $\mathcal{M}$'s state space. 
Let $\hat{\theta}^{\mathtt{rb}}_{N}$ be the Rao-Blackwell estimator for $N$ sample points $s_{T+1}, ..., s_{T+N}$ collected from $\mathcal{M}$, after discarding the first $T$ sample points. 
Then, $|\mathtt{bias}(\hat{\theta}^{\mathtt{rb}}_{N})| \leq \epsilon$ if  $T \geq \tau'(\epsilon)$, where $\tau'(\epsilon)$ is the mixing time of the quotient Markov chain of $\mathcal{M}$ with respect to $\mathcal{O}$.
\end{theorem}

Hence, if one wants to make sure that the absolute value of the bias of the Rao-Blackwell estimator is smaller than a given $\epsilon > 0$, one only needs a burn-in period consisting of $\tau'(\epsilon)$ simulation steps, where $\tau'(\epsilon)$ is the mixing time of the quotient Markov chain. Existing work on analyzing the influence of symmetries in random walks has shown that it is often more convenient to investigate the mixing behavior of the quotient Markov chain~\cite{boyd:2005}. In the context of marginal density estimation, Markov chains implicitly operating on the orbit partition of the assignment space were shown to have better mixing behavior~\cite{niepert2012omcmc}. 

In summary, whenever probabilistic models exhibit non-trivial symmetries we can have the best of both worlds. The bias owed to the fact that we are collecting a finite number of sample points from a Markov chain as well as the asymptotic variance~\cite{neal:2004} of the Rao-Blackwell estimator are at least as small as those of the standard estimator. The more symmetric the  probabilistic model the larger the reduction in mean squared error.

We now present the experimental results for several large probabilistic models, both relational and non-relational.

\section{Experiments}

The aim of the empirical investigation is twofold. First, we want to verify the efficiency of the novel Rao-Blackwell estimator when applied as a post-processing step to the output of state-of-the-art sampling algorithms. Second, we want to test the hypothesis that the efficiency gains of the novel estimator on standard SRL models are similar empirically to those of state-of-the-art lifted inference algorithms.

For the SRL models we computed the orbit partitions of the variables based on the model's renaming automorphisms~\cite{hung:2012}. As discussed earlier, renaming automorphisms are computable in time linear in the domain size. We applied \textsc{Gap}~\cite{GAP4} to compute the variables' orbit partition. 
For all non SRL models we computed the automorphism group and the orbit partitions as in \cite{niepert2012omcmc} using the graph automorphism algorithm \textsc{Saucy}~\cite{darga:2008} and the \textsc{Gap} system, respectively. Overall, the computation of the orbit partitions  of the models' variables took less than one second for each of the probabilistic models we considered.

\begin{figure*}[t!]
     \begin{center}
        \subfigure[Average time for the smokes-cancer MLN.]{%
            \label{fig:first2}
            \includegraphics[width=0.23\textwidth]{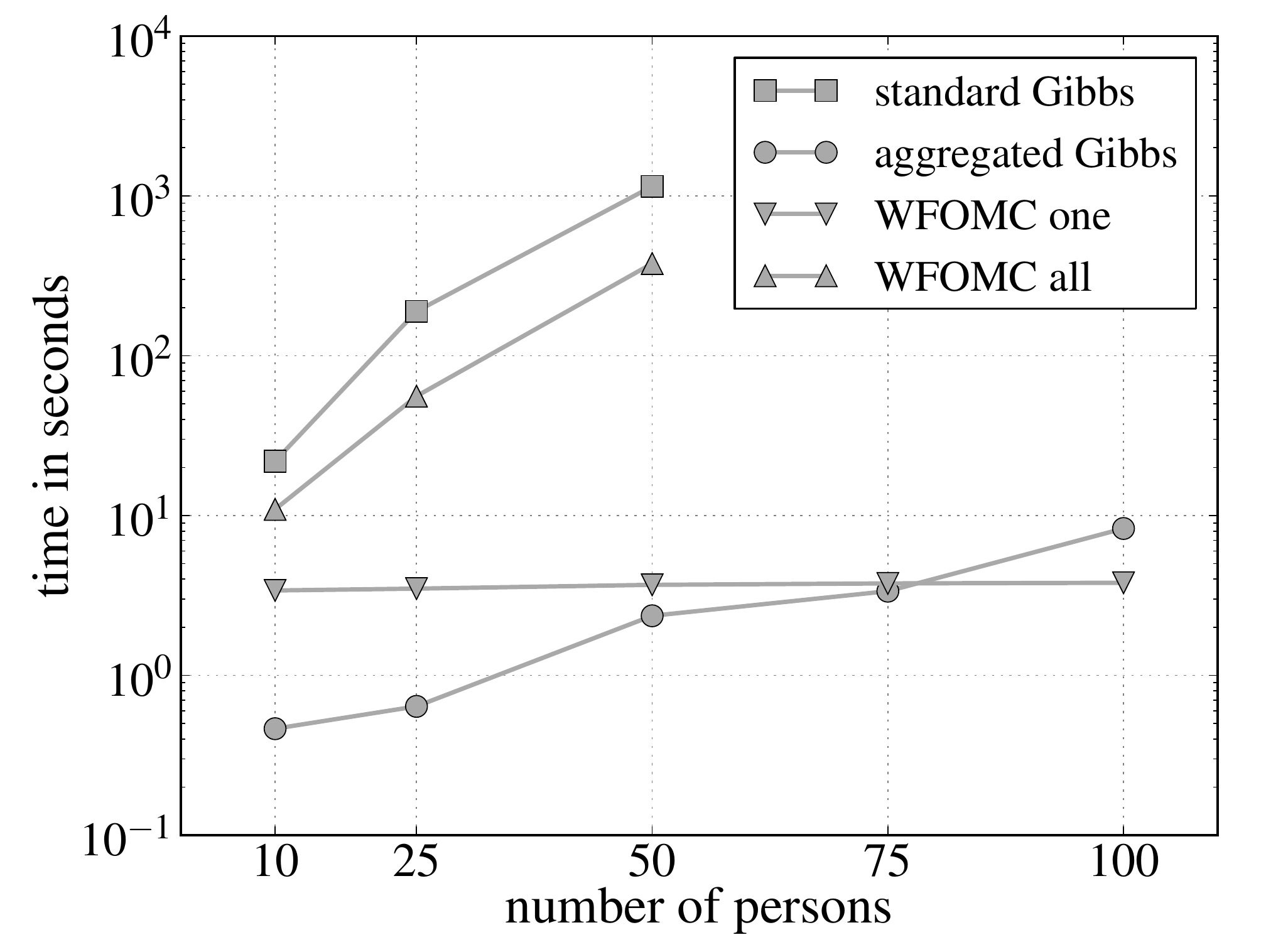}
        }%
        \hspace{1.2mm}
        \subfigure[Number of sample points for the smokes-cancer MLN.]{%
           \label{fig:second2}
           \includegraphics[width=0.23\textwidth]{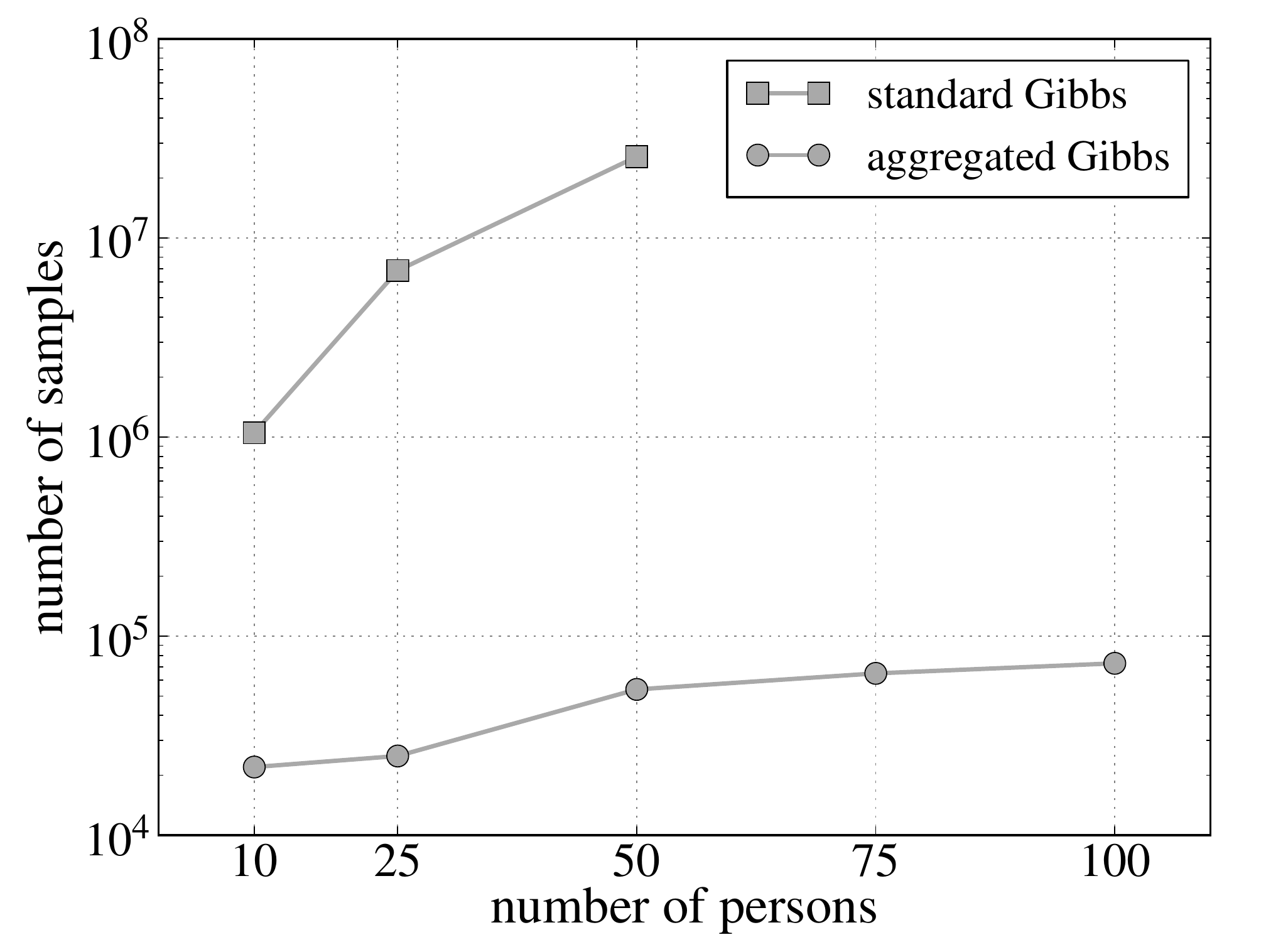}
        }  \hspace{1.2mm}
        \subfigure[Average time for the smokes-cancer MLN with transitivity.]{%
            \label{fig:third2}
            \includegraphics[width=0.23\textwidth]{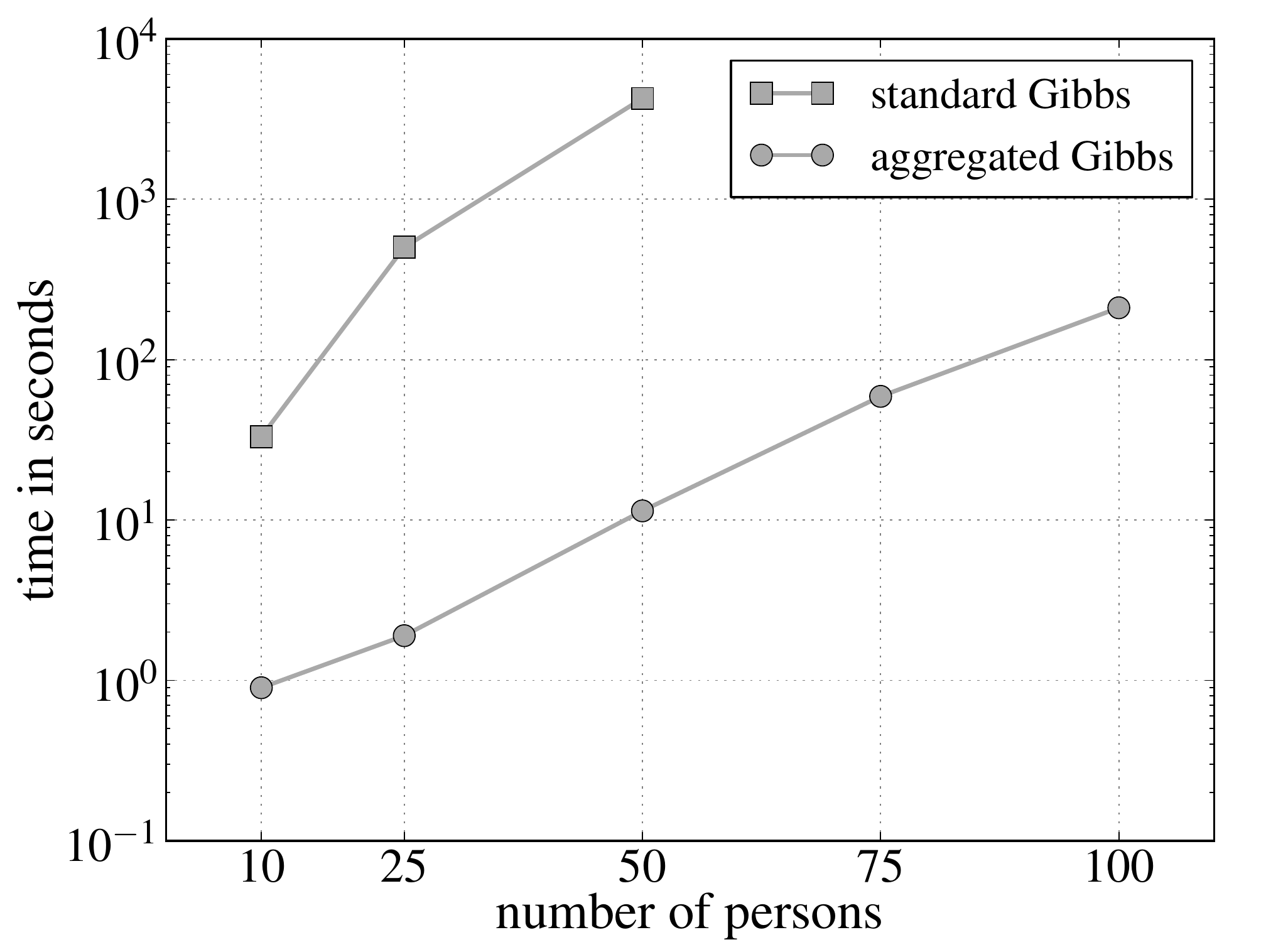}
        }%
        \hspace{1.2mm}
        \subfigure[Sample points for the  smokes-cancer MLN with transitivity.]{%
            \label{fig:fourth2}
            \includegraphics[width=0.23\textwidth]{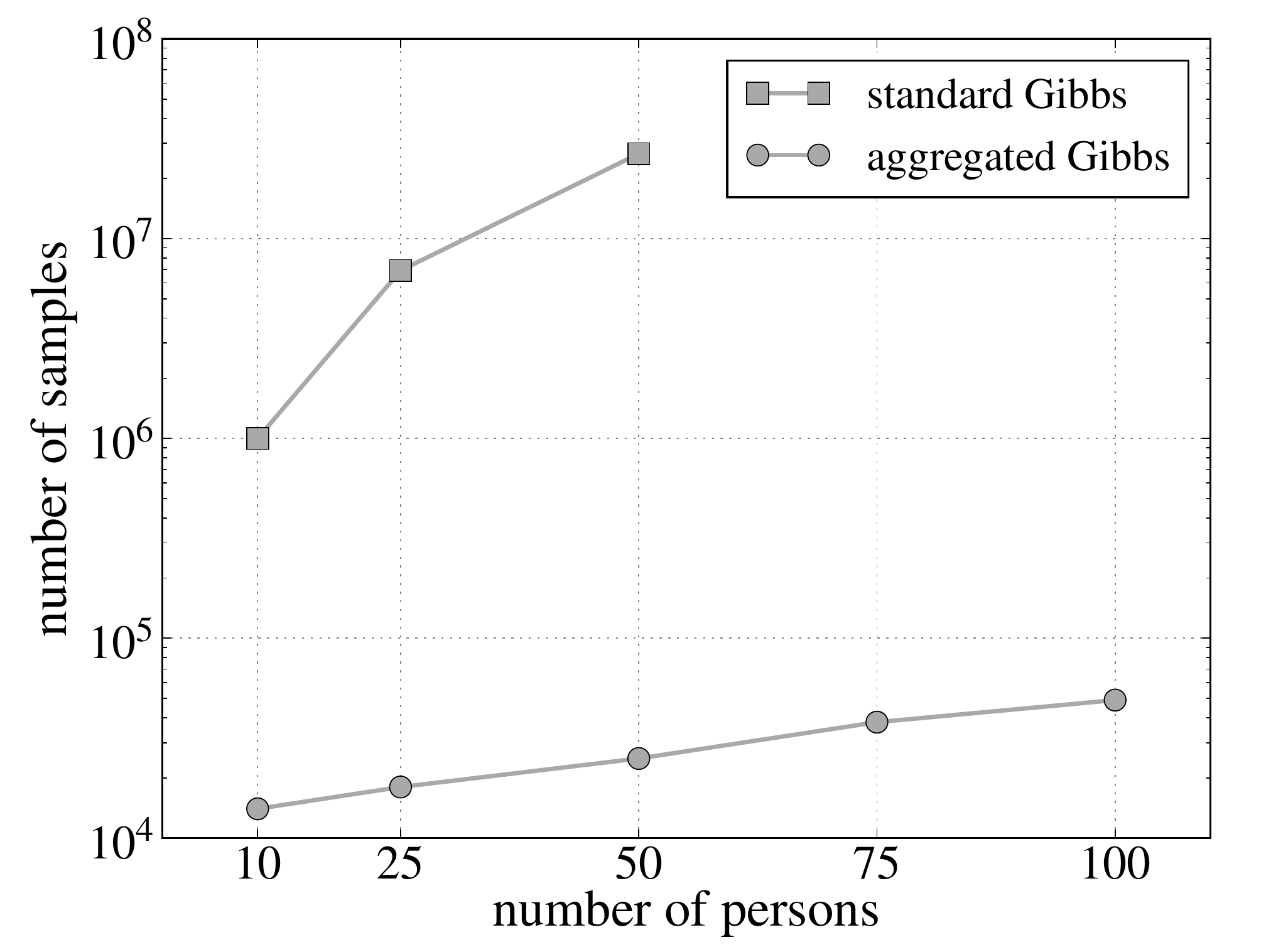}
        }%
    \end{center}
    \caption{\label{fig:MLNs-domain} Average time and number of sample points, respectively, needed to achieve an average KL divergence of $<0.0001$.}
\end{figure*}

We conducted experiments with several benchmark Markov logic networks, a statistical relational language general enough to capture numerous types of  graphical models~\cite{RD:2006}. Here, we used (a) the asthma-smokes-cancer MLN~\cite{venugopal:2012} with 10\% evidence\footnote{For a random 10\% of all people it is known (a) whether they smoke or not and (b) who 10 of their friends are.}; (b) the ``Friends \& Smokers" MLN exactly as specified in \cite{singla2008lifted} with 10\% evidence; and the ``Friends \& Smokers" MLN with the transitivity formula on the \texttt{friends} relation having weight $1.0$, (c) without and (d) with 10\% evidence. 
Each of the models had between $10$ and $100$ objects in the domain, leading to log-linear models with $10^2$-$10^4$ variables and $10^2$-$10^6$ features. We used \textsc{Wfomc}~\cite{broeck:2011}, to compute the exact single-variable marginals of the asthma MLN. For all other MLNs, existing exact lifted inference algorithm were unable to compute single-variable densities. In these cases, we performed several very long runs (burn-in $1$ day; overall $5$ days) of a Gibbs sampler guaranteed to be ergodic and made sure that state-of-the-art MCMC diagnostics indicated convergence~\cite{brooks:1998}.  

We executed our implementation of the standard Gibbs sampler and \textsc{Alchemy}'s implementation of the  MC-SAT algorithm~\cite{Poon:2006} on the MLNs based on $10$ separate runs, without a burn-in period. For each sampling algorithm we computed the marginal densities with the standard estimator and the Rao-Blackwell estimator, respectively, which we implemented in the \textsc{Gap} programming language\footnote{https://code.google.com/p/lifted-mcmc/}.  Figure~\ref{fig:MLNs-comparative} depicts, for each MLN, the average Kullback-Leibler divergence\footnote{We computed both MSE and average KL divergence but omitted the qualitatively identical MSE results due to space constraints.} between the estimated and precomputed true single-variable marginals of the non-evidence variables plotted against the absolute running time of the algorithms in seconds.

The Rao-Blackwell estimator improves the density estimates by at least an order of magnitude and, in the absence of evidence, even up to four orders of magnitude relative to the standard estimator. The improvement of the empirical results is independent of the relational structure of the MLNs. For the MLN with a transitivity formula on the \texttt{friends} relation, generally considered a problematic and as of now not domain-liftable model, the results are as pronounced as for the MLNs known to be domain-liftable. 

We also conducted experiments with non-SRL models to investigate the efficiency of the approach on graphical models. We executed the Gibbs sampler with and without using the Rao-Blackwell estimator on a $100\times100$ $2$-coloring grid model with binary random variables. The symmetries of the model are the reflection and rotation automorphisms of the $2$-dimensional square grid.  Figure~\ref{fig:MLNs-comparative}(e) depicts the plot of the average KL divergence against the running time in seconds, where each pairwise factor between neighboring variables $X, Y$ was defined as $\exp(0.2)$ if $X\neq Y$, and $1$ otherwise. Figure~\ref{fig:MLNs-comparative}(f) depicts the plot of the same grid model except that the pairwise factors were defined as $1$ if $X\neq Y$, and $0$ otherwise. The results clearly demonstrate the superior performance of the Rao-Backwell estimator even for probabilistic models with a smaller number of symmetries.

In addition, we analyzed the impact of the domain size on the estimator performance for (a) domain-liftable MLNs and (b) MLNs not liftable by any state-of-the-art exact lifted inference algorithm. We used the ``Friends \& Smokers" MLN without evidence; and with and without the transitivity formula on the \texttt{friends} relation. The MLN without transitivity is a standard benchmark for lifted algorithms whereas MLNs with transitivity are considered difficult and no exact lifted inference algorithm exists for such MLNs as of now. Figures~\ref{fig:MLNs-domain}(a)\&(c) depict the time needed to achieve an average KL divergence of less than $10^{-4}$ plotted against the domain size of the models without and with transitivity.  The increase in runtime is far less pronounced with the Rao-Blackwell estimator. The plots resemble those often shown in lifted inference papers where an algorithm that can lift a model is contrasted with one that cannot. The increase in runtime is slightly higher for the model with transitivity but this is owed to the size increase of each variable's Markov blanket and, thus, the time needed for each Gibbs sampler step. Figures~\ref{fig:MLNs-domain}(b)\&(d) plots the sample size required to achieve an average KL divergence of less than $10^{-4}$ against the domain size. Interestingly, the number of sample points is almost identical for the model with and without transitivity, indicating that the advantage of the Rao-Blackwell estimator is independent of the model's formulas. 

In Figure~\ref{fig:MLNs-domain}(a) we plot the results of \textsc{Wfomc} for compiling a first-order circuit and computing (a) one single-variable marginal and (b) all single-variable marginals. \textsc{Wfomc} has constant runtime for exactly computing one single-variable marginal density. The Rao-Blackwell  estimation for all of the model's variables scales sub-linearly and is more efficient than repeated calls to \textsc{Wfomc}. While we do not need to run \textsc{Wfomc} once per single-variable marginal density  if the variables are first partitioned into sets of variables with identical marginal densities~\cite{braz:2005}, the results  demonstrate that the symmetry-aware estimator scales comparably to exact lifted inference algorithms on domain-liftable models and that its runtime is polynomial in the domain size of the MLNs.

\section{Discussion}

A Rao-Blackwell estimator was developed and shown, both analytically and empirically, to have lower mean squared error under non-trivial model symmetries. The presented theory provides a novel perspective on the notion of lifted inference and the underlying reasons for the feasibility of marginal density estimation in large but highly symmetric probabilistic models. For the first time, the applicability of such an approach does not directly depend on the properties of the relational structure such as the arity of predicates and the type of formulas but  only  on the given \emph{evidence} and the corresponding automorphism group of the model. 
We believe the theoretical and empirical insights to be of great interest to the machine learning community and that the presented work might contribute to a deeper understanding of lifted inference algorithms.


\section{Acknowledgments}

Many thanks to Guy Van den Broeck who provided feedback on an earlier draft of the paper. This work was partially supported by a Google faculty research award.

\bibliographystyle{aaai}
\bibliography{aaai13}

\appendix
\newpage

\section{Proof of Lemma~\ref{lemma-closed-form}}

\begin{lemmaappendix}
Let $\mathbf{X}$ be a finite sequence of random variables with joint distribution $P(\mathbf{X})$, let $\mathfrak{G}$ be an automorphism group of $\mathbf{X}$,  let $\mathbf{\hat{X}}$ be a subsequence of $\mathbf{X}$, let ${\mathbf{\hat{X}}}^\mathfrak{G}$ be the orbit of $\mathbf{\hat{X}}$, and let $s \in \mathsf{Val}(\mathbf{X})$. Then,  
$$P(\mathbf{\hat{X}}=\mathbf{\hat{x}}\ |\ s^{\mathfrak{G}}) = \frac{\mathsf{H}_{\mathbf{\hat{X}}=\mathbf{\hat{x}}}^{\mathfrak{G}}(s)}{|\mathbf{\hat{X}}^\mathfrak{G}|} = \mathbb{E}[ \hat{\theta}_{N}\ |\ \mathsf{H}_{\mathbf{\hat{X}}=\mathbf{\hat{x}}}^{\mathfrak{G}}(s) ].$$
\end{lemmaappendix}

\begin{proof}
Let $\mathfrak{G}_s := \{ \mathfrak{g} \in \mathfrak{G}\ |\ s^{\mathfrak{g}}=s\}$ be the stabilizer subgroup of $s$. Then, $$P(\mathbf{\hat{X}}=\mathbf{\hat{x}}\ |\ s^{\mathfrak{G}}) = \frac{|\{\mathfrak{g} \in \mathfrak{G}\ |\ s^{\mathfrak{g}}\langle\mathbf{\hat{X}}\rangle=\mathbf{\hat{x}}\}|}{|\mathfrak{G}|}$$ since for each $x \in s^{\mathfrak{G}}$ we have that $|\{ \mathfrak{g} \in \mathfrak{G}\ |\ s^{\mathfrak{g}} = x\}| = |\mathfrak{G}_s|$ by the orbit stabilizer theorem. For each $\mathbf{A} \in {\mathbf{\hat{X}}}^\mathfrak{G}$ let $\mathfrak{G}^{\mathbf{A}} := \{ \mathfrak{g} \in \mathfrak{G}\ |\ \mathbf{A}^{\mathfrak{g}}=\mathbf{\hat{X}}\}$. Again, by the orbit stabilizer theorem, we have that $|\mathfrak{G}^{\mathbf{A}}| = |\mathfrak{G}_{\mathbf{\hat{X}}}|$ for each $\mathbf{A} \in {\mathbf{\hat{X}}}^\mathfrak{G}$, where $\mathfrak{G}_{\mathbf{\hat{X}}}$ is the stabilizer subgroup of $\mathbf{\hat{X}}$. Hence,  $$\frac{|\{\mathfrak{g} \in \mathfrak{G}\ |\ s^{\mathfrak{g}}\langle\mathbf{\hat{X}}\rangle=\mathbf{\hat{x}}\}|}{|\mathfrak{G}|} = \frac{\mathsf{H}_{\mathbf{\hat{X}}=\mathbf{\hat{x}}}^{\mathfrak{G}}(s)|\mathfrak{G}_{\mathbf{\hat{X}}}|}{|\mathfrak{G}|} = \frac{\mathsf{H}_{\mathbf{\hat{X}}=\mathbf{\hat{x}}}^{\mathfrak{G}}(s)}{|\mathbf{\hat{X}}^\mathfrak{G}|}.$$
Hence, $\mathsf{H}_{\mathbf{\hat{X}}=\mathbf{\hat{x}}}^{\mathfrak{G}}$ is a sufficient statistic for the marginal density $P(\mathbf{\hat{X}}=\mathbf{\hat{x}})$. Moreover, we have that $$P(\mathbf{\hat{X}}=\mathbf{\hat{x}}\ |\ s^{\mathfrak{G}}) = \mathbb{E}[ \hat{\theta}_{N}\ |\ \mathsf{H}_{\mathbf{\hat{X}}=\mathbf{\hat{x}}}^{\mathfrak{G}}(s) ].$$
This concludes the proof.
\end{proof}

\section{Proof of Theorem~\ref{theorem-rao-blackwell}}

\begin{theoremappendix}
Let $\mathbf{X}$ be a finite sequence of random variables with joint distribution $P(\mathbf{X})$, let $\mathfrak{G}$ be an automorphism group of $\mathbf{X}$ given by $R$ generators, let $\mathbf{\hat{X}}$ be a subsequence of $\mathbf{X}$, let $\mathbf{\hat{X}}^\mathfrak{G}$ be the orbit of $\mathbf{\hat{X}}$, and let $\theta := P(\mathbf{\hat{X}}=\mathbf{\hat{x}})$ be the marginal density to be estimated. The Rao-Blackwell estimator $\hat{\theta}^{\mathtt{rb}}_{N}$ has the following properties:
\begin{enumerate}
{\setlength\itemindent{12pt} \item[(a)] Its worst-case time complexity is $O(R|\mathbf{\hat{X}}^\mathfrak{G}|\hspace{-0.5mm} +\hspace{-0.5mm} N |\mathbf{\hat{X}}^\mathfrak{G}|)$; }
{\setlength\itemindent{12pt} \item[(b)] $\mbox{MSE}[\hat{\theta}^{\mathtt{rb}}_{N}] \leq \mbox{MSE}[\hat{\theta}_N]$. }
\end{enumerate} 
The inequality of (b) is strict if there exists a joint assignment $s$ with non-zero density and $0 < \mathsf{H}_{\mathbf{\hat{X}}=\mathbf{\hat{x}}}^{\mathfrak{G}}(s) <|\mathbf{\hat{X}}^{\mathfrak{G}}| > 1$.
\end{theoremappendix}

\begin{proof}
We first construct the set $\mathbf{\hat{X}}^\mathfrak{G}$ once, which has a worst-case time complexity of $R|\mathbf{\hat{X}}^\mathfrak{G}|$~\cite{holt:2005}. For each sample point, we have to access an array representing the values of the sample point at most $|\mathbf{\hat{X}}^\mathfrak{G}|$ times. This allows us, for each sample point $s$, to compute $P(\mathbf{\hat{X}}=\mathbf{\hat{x}}\ |\ s^{\mathfrak{G}})$ in time $O(|\mathbf{\hat{X}}^\mathfrak{G}|)$ by Lemma~\ref{lemma-closed-form}.

Since $\mathsf{H}_{\mathbf{\hat{X}}=\mathbf{\hat{x}}}^{\mathfrak{G}}$ is a sufficient statistic for $\theta$ and $\hat{\theta}^{\mathtt{rb}}_{N} = \mathbb{E}[\hat{\theta}_N\ |\ \mathsf{H}_{\mathbf{\hat{X}}=\mathbf{\hat{x}}}^{\mathfrak{G}}]$ by Lemma~\ref{lemma-closed-form}, statement (b) follows from the Rao-Blackwell theorem~\cite{blackwell:1947}. If there exists a joint assignment $s$ with non-zero density and $0 < \mathsf{H}_{\mathbf{\hat{X}}=\mathbf{\hat{x}}}^{\mathfrak{G}}(s) <|\mathbf{\hat{X}}^\mathfrak{G}| > 1$, then $\hat{\theta}_N$ is not a function of $\mathsf{H}_{\mathbf{\hat{X}}=\mathbf{\hat{x}}}^{\mathfrak{G}}$ and the inequality is strict~\cite{blackwell:1947}.
\end{proof}

\newpage

\section{Proof of Theorem~\ref{theorem-mcmc}}

\begin{theoremappendix}
Let $\mathbf{X}$ be a finite sequence of random variables with joint distribution $P(\mathbf{X})$, let $\mathcal{M}$ be an ergodic Markov chain with stationary distribution $P$, and let $\mathcal{O}$ be an orbit partition of $\mathcal{M}$'s state space. 
Let $\hat{\theta}^{\mathtt{rb}}_{N}$ be the Rao-Blackwell estimator for $N$ sample points $s_{T+1}, ..., s_{T+N}$ collected from $\mathcal{M}$, after discarding the first $T$ sample points. 
Then, $|\mathtt{bias}(\hat{\theta}^{\mathtt{rb}}_{N})| \leq \epsilon$ if  $T \geq \tau'(\epsilon)$, where $\tau'(\epsilon)$ is the mixing time of the quotient Markov chain of $\mathcal{M}$ with respect to $\mathcal{O}$.
\end{theoremappendix}

\begin{proof}
For a subsequence $\mathbf{\hat{X}}$ of $\mathbf{X}$, let $\mathbf{\xi} := \mathbf{\hat{X}}=\mathbf{\hat{x}}$ be the marginal assignment whose density $\theta$ is to be estimated, let $\mathsf{Val}(\mathbf{X})$ be the assignment space of  $\mathbf{X}$, and let $S = \{s_{T+1}, ..., s_{T+N}\}$ be the multiset of sample points collected from $\mathcal{M}$, after discarding the first $T$ sample points.
Since $\mathcal{O}$ is a partition of the assignment space, we have that  $$\hat{\theta}^{\mathtt{rb}}_{N} = \frac{1}{N}\sum_{s \in S}P(\mathbf{\xi}\ |\ s^{\mathfrak{G}}) = \sum_{O \in \mathcal{O}} P(\mathbf{\xi}\ |\ O) \frac{1}{N} \sum_{s \in S} \mathbb{I}_{\{s \in O\}} .$$  
Hence,$$\mathbb{E}[\hat{\theta}^{\mathtt{rb}}_{N}] = \sum_{O \in \mathcal{O}}P(\mathbf{\xi}\ |\ O) \mathbb{E}[ \mathbb{I}_{\{s \in O\}}],$$ where $\mathbb{E}[ \mathbb{I}_{\{s \in O\}}]$ is the expectation of some sample point being located in the orbit $O$. $\mathbb{E}[ \mathbb{I}_{\{s \in O\}}]$ defines a probability distribution over the space $\mathcal{O}$. If the sample points are independent, then $\mathbb{E}[ \mathbb{I}_{\{s \in O\}}] = P(O)$, for all $O\in\mathcal{O}$, and the estimator is unbiased. Since we collect sample points from a Markov chain we will often have that $\mathbb{E}[ \mathbb{I}_{\{s \in O\}}] \neq P(O)$.

By the assumptions and Proposition~\ref{proposition-lumpable}, the Markov chain $\mathcal{M}$ is exactly lumpable with respect to $\mathcal{O}$ and, hence, for all states $x$  of $\mathcal{M}$, all $t \in \{1, 2, ...\}$, and all orbits $O \in \mathcal{O}$, we have that $\sum_{o \in O}Q^t(x, o) = {Q'}^t(x^{\mathfrak{G}}, O)$, where ${Q}^t(x, o)$ is the probability of the Markov chain $\mathcal{M}$ being in state $o$ after $t$ simulation steps  if the chain starts in state $x$.
In addition, we start collecting sample points after $T \geq \tau'(\epsilon)$ simulation steps and, thus, $$\frac{1}{2}\sum_{O \in \mathcal{O}} |\mathbb{E}[\mathbb{I}_{\{s \in O\}}] - P(O)| \leq$$ $$\max_{x \in \mathsf{Val}(\mathbf{X})}{\left\lbrace\frac{1}{2}\sum_{O \in \mathcal{O}} |{Q'}^T(x^{\mathfrak{G}}, O) - P(O)|\right\rbrace} \leq \epsilon.$$

\noindent Finally, $|\mathtt{bias}(\hat{\theta}^{\mathtt{rb}}_{N})| = |\mathbb{E}[\hat{\theta}^{\mathtt{rb}}_{N} - \theta]| =$
$$\left|\sum_{O \in \mathcal{O}}P(\mathbf{\xi}\ |\ O) \mathbb{E}[\mathbb{I}_{\{s \in O\}}]  - \sum_{O \in \mathcal{O}} P(\mathbf{\xi}\ |\ O) P(O) \right| =$$
$$\left|\sum_{O\in \mathcal{O}}P(\mathbf{\xi}\ |\ O) (\mathbb{E}[\mathbb{I}_{\{s \in O\}}] - P(O))\right| \leq $$
$$\sum_{\substack{O \in \mathcal{O}\\ \mathbb{E}[\mathbb{I}_{\{s \in O\}}] \geq P(O)}} (\mathbb{E}[\mathbb{I}_{\{s \in O\}}] - P(O)) =$$
$$\frac{1}{2}\sum_{O \in \mathcal{O}} |\mathbb{E}[\mathbb{I}_{\{s \in O\}}] - P(O)| \leq \epsilon.$$
The last equality follows from a known identity of the total variation distance~\cite{levin:2008}.
\end{proof}

\end{document}